\documentclass{article}
\usepackage{spconf,amsmath,graphicx}
\usepackage{array}
\usepackage{color}
\usepackage{multirow}
\usepackage{graphicx}
\graphicspath{{fig/}}
\usepackage{hyperref}

\usepackage{amsthm}

\newtheorem{prop}{Proposition}

\newcommand{\figref}{\figurename~\ref}
\newcommand{\equref}{\eqref}
\newcommand{\tabref}{\tablename~\ref}
\newcommand{\secref}{Session \ref}

\title{Improved training of neural trans-dimensional random field language models with dynamic noise-contrastive estimation}
\name{Bin Wang, Zhijian Ou\thanks{This work is supported by NSFC grant 61473168.}}
\address{Speech Processing and Machine Intelligence (SPMI) Lab, Tsinghua University, Beijing, China. \\
         wangbin12@mails.tsinghua.edu.cn, ozj@tsinghua.edu.cn}
\begin{document}
\ninept
\maketitle
\begin{abstract}
A new whole-sentence language model - neural trans-dimensional random field language model (neural TRF LM), where sentences are modeled as a collection of random fields, and the potential function is defined by a neural network, has been introduced and successfully trained by noise-contrastive estimation (NCE). In this paper, we extend NCE and propose dynamic noise-contrastive estimation (DNCE) to solve the two problems observed in NCE training. First, a dynamic noise distribution is introduced and trained simultaneously to converge to the data distribution. This helps to significantly cut down the noise sample number used in NCE and reduce the training cost. Second, DNCE discriminates between sentences generated from the noise distribution and sentences generated from the interpolation of the data distribution and the noise distribution. This alleviates the overfitting problem caused by the sparseness of the training set. With DNCE, we can successfully and efficiently train neural TRF LMs on large corpus (about 0.8 billion words) with large vocabulary (about 568 K words). Neural TRF LMs perform as good as LSTM LMs with less parameters and being 5x$\sim$114x faster in rescoring sentences. Interpolating neural TRF LMs with LSTM LMs and n-gram LMs can further reduce the error rates.
\end{abstract}
\section{Introduction}
\label{sec:intro}

Statistical language models (LMs), which estimate the  probability of a sentence, are an important component in various applications, such as automatic speech recognition (ASR) and machine translation (MT).
Most LMs, including the classical n-gram LMs \cite{chen1999empirical} and the recurrent neural network LMs \cite{mikolov2011}, follow the directed graphical modeling approach, where the probability of a sentence is calculated as the product of local conditionals.
Recently, there are increasing interests in investigating whole-sentence LMs \cite{Bin2015,Bin2017,Bin2017ASRU,Bin2018,huang2018whole}, which directly model the joint probability of a whole sentence without local normalizations.
Typically, trans-dimensional random field (TRF) LMs \cite{Bin2015,Bin2017,Bin2017ASRU,Bin2018} are proposed in the undirected graphical modeling approach,
where sentences are modeled as a collection of random fields and the sentence probability is defined in term of potential functions.
TRF LMs can flexibly support any types of discrete or neural network features.
Specifically, the neural TRF LMs \cite{Bin2017ASRU, Bin2018}, whose potential function is defined by a neural network, have been shown to outperform the classical n-gram LMs significantly, and perform close to LSTM LMs but are computational more efficient in computing sentence probabilities.

Training neural TRF LMs is challenging, especially on large corpus with large vocabulary.
In \cite{Bin2018}, noise-contrastive estimation (NCE) \cite{nce} is introduced to train neural TRF LMs,
by discriminating between real sentences drawn from the data distribution and noise sentences generated from a noise distribution.
However, the NCE training is found to have the following two problems.
First, reliable NCE needs to generate a large number of noise sentences from the noise distribution, especially when the noise distribution is not close to the data distribution.
However, the time and memory cost for gradient calculation are almost linearly increased with the number of noise samples.
In \cite{Bin2018}, the noise distribution is defined by a bigram LM, which is far from the data distribution. For each real sentence, 20 noise sentences are generated from the bigram LM, which is highly undesirable.
Second, the consistency property of NCE holds when an arbitrarily large number of real sentences could be drawn from the true but unknown data distribution.
In practice, real sentences are sampled from the empirical distribution (namely the training set), which is rather sparse considering the high-dimensionality of sentences.
The model estimated by NCE is thus easily overfitted to the empirical distribution.
Due to the two problems, the neural TRF LMs in \cite{Bin2018} are defined in the form of exponential tilting of a reference LSTM LM and consequently loss the advantage of the efficient inference.

In this paper, we propose dynamic noise-contrastive estimation (DNCE), which consists of two extensions beyond of the original NCE algorithm to address the above two problems respectively and thus significantly improves the training of neural TRF LMs.
First, a dynamic noise distribution is introduced and trained simultaneously by minimizing the Kullback-Leibler (KL) divergence between the noise distribution and the data distribution.
With a noise distribution that is close to the data distribution, NCE can achieve reliable model estimation even using a small number of noise sentences.
Second, DNCE discriminates between noise sentences generated from the dynamic noise distribution and sentences generated from the interpolation of the data distribution and the noise distribution.
Intuitively, this increases the size of training set by adding noise sentences (which are asymptotically distributed according to the data distribution) and alleviates the overfitting problem caused by the sparseness of the training set.

Three speech recognition experiments are conducted to evaluate the neural TRF LMs with DNCE training.
First, various LMs are trained on Wall Street Journal (WSJ) portion of Penn Treebank (PTB) English dataset and then used to rescore the 1000-best list generated from the WSJ'92 test set, with the same experimental setup as in \cite{Bin2017, Bin2017ASRU}.
Then LMs are evaluated in the speech recognition experiment on HKUST Chinese dataset \cite{hkust}.
The above two experiments demonstrate the language independence in applying neural TRF LMs.
The neural TRF LMs outperform the classical 5-gram LMs significantly, and perform as good as the LSTM LMs but are computational more efficient (5x to 114x faster) than LSTM LMs even when the vocabulary size is not large (4 K to 10 K).
Finally, to evaluate the scalability of neural TRF LMs and DNCE, we conduct the experiment on the Google one-billion benchmark dataset \cite{google1b}, which contains about 0.8 billion training words with a vocabulary of about 568 K words.
Compared to a large LSTM LM with adaptive softmax \cite{grave2016efficient}, the neural TRF LM achieves a slightly lower WER and is also 54x faster in rescoring the n-best list.
Moreover, combing the neural TRF LMs with LSTM LMs and n-gram LMs can further reduce the error rates.
The source codes of all the experiments can be obtained in \url{https://github.com/wbengine/TRF-NN-Tensorflow}

The rest of the paper is organized as follows.
We discuss related work in \secref{sec:relate} and present basics about neural TRF LMs and NCE in \secref{sec:background}.
The proposed DNCE method is described in \secref{sec:dnce}.
After presenting the three experiments in \secref{sec:exps}, the conclusions are made in \secref{sec:conclusion}.

\section{Related work}
\label{sec:relate}

The NCE method is first proposed in \cite{nce}, and has been used to train conditional neural network (NN) LMs, such as the feedforward NN LMs \cite{vaswani2013decoding} and LSTM LMs \cite{zoph2016simple},
to avoid the expensive softmax calculation by treating the learning as a binary classification problem between the target words and the noise samples.
As storing all the context-dependent normalization constants of NN LMs is infeasible, an approximation is to freeze them to an empirical values, which is 1 in \cite{zoph2016simple,sethy2015unnormalized} and $e^9$ in \cite{chen2015recurrent}.
With the aim to fully utilize dense matrix operations, some studies propose to share the noise samples between target words \cite{zoph2016simple,oualil2017batch}.
Different from the above NCE related studies, DNCE improves NCE in general by introducing a dynamic noise distribution and using the interpolation of the data distribution and the dynamic noise distribution in the discriminator.
Compared with using hundreds of noise samples per data sample in \cite{oualil2017batch}, DNCE  uses at most 4 noise samples per data sample in our experiments.

The idea of whole-sentence LMs is first proposed in \cite{rosenfeld1997whole}, called whole-sentence maximum entropy (WSME) LMs, and further studied in \cite{Chen1999Efficient, rosenfeld2001whole, amaya2001improvement, ruokolainen2010using}.
The empirical results of these previous WSME LMs are not satisfactory, almost the same as traditional n-gram LMs.
Recently, \cite{huang2018whole} follows WSMEs to propose a whole-sentence neural LMs, which use neural network features and NCE training.
However in \cite{huang2018whole}, noise sample generation and log-likelihood evaluation from the noise distribution are not matched but empirically found to work well; 20 noise samples per data sample are generated; the whole sentence neural LMs are trained on a small corpus with at most 24 M words, and interpolated with LSTM LMs; the performance of the whole sentence neural LMs alone is not reported and compared.
Though with encouraging results, the above issues may adversely affect the whole sentence neural LMs.

Different from the class of WSME LMs, a TRF LM is defined as a mixture of random fields for joint modeling sentences of different dimensions (namely lengths), with mixture weights explicitly as the priori length probabilities (See \cite{Bin2017} for detailed comparison).
Hence the model is called a trans-dimensional random field (TRF).
It is worthwhile to review the development of TRF LMs \cite{Bin2015,Bin2017,Bin2017ASRU,Bin2018}.
TRF LMs are first proposed in \cite{Bin2015} and further presented in \cite{Bin2017} with thorough theoretical analysis and systematic evaluation. Both \cite{Bin2015} and \cite{Bin2017} use discrete features and employs the augmented stochastic approximation (AugSA) for model training.
Neural TRF LMs are proposed in \cite{Bin2017ASRU}, by defining the potential function as a deep convolutional neural network (CNN). Model training is performed by AugSA plus JSA (joint stochastic approximation), which introduces an auxiliary distribution to improve the sampling process in AugSA.
In \cite{Bin2018}, NCE is introduced to train TRF LMs, and CNN and LSTM are married to define the potential function.
In this paper, we propose DNCE for improved model training and also simplify the potential definition by using only the bidirectional LSTM.


\section{Background}
\label{sec:background}

\subsection{Neural trans-dimensional random field LMs}
\label{sec:trf}

As in \cite{Bin2018}, the joint probability of a sequence $x^l$ and its length $l$ is assumed to be distributed from an exponential family model:
    \begin{equation}\label{eq:trf}
    p(l, x^l;\theta) = \frac{\pi_l}{Z_l(\theta)} e^{\phi(x^l;\theta)}
    \end{equation}
where $x^l=(x_1, \ldots, x_l)$ is a word sequence of length $l$ ($l=1, \ldots, m$),
$\pi_l$ is the prior length probability,
$\theta$ indicates the set of parameters,
$Z_l(\theta)$ is the normalization constant of length $l$, i.e. $Z_l(\theta) = \sum_{x^l} e^{\phi(x^l; \theta)}$.
$\phi$ is the potential function, which can be defined by neural networks.

In this paper, different from the model definitions in \cite{Bin2017ASRU, Bin2018}, we simplify the neural network architecture and define the potential function by a bidirectional LSTM as shown in \figref{fig:rnn-trf}.
Compared with the bidirectional LSTM LMs in \cite{chen2017investigating, he2016training}, this neural TRF LM provides a theoretical-solid framework to incorporate the bidirectional LSTM features.

\begin{figure}
	\centering
	\includegraphics[width=0.9\linewidth]{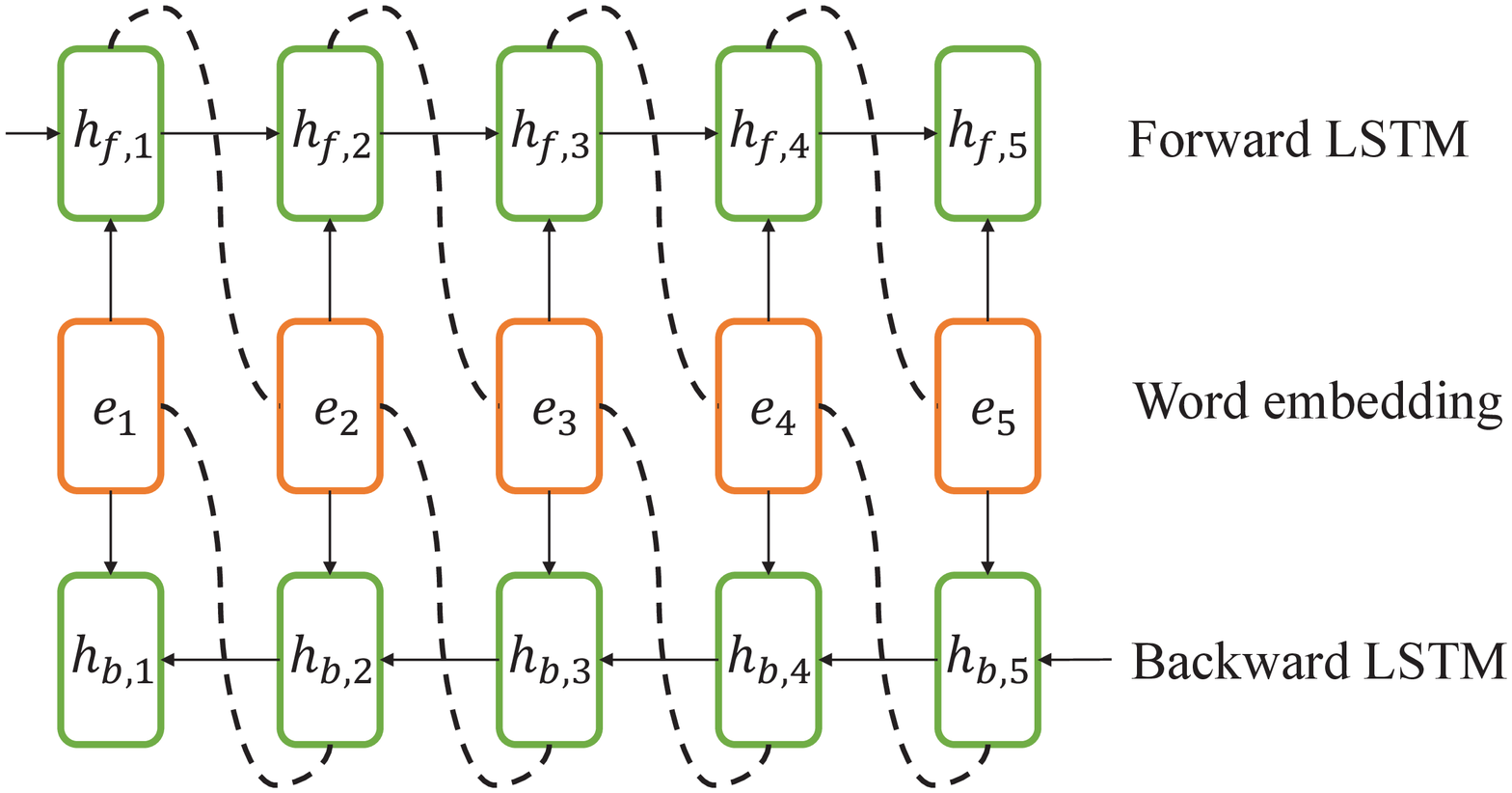}
	\caption{The bidirectional LSTM used to define the potential function  $\phi(x^l;\theta)$ in neural TRFs.}\label{fig:rnn-trf}
	\vskip -0.1in
\end{figure}

The bidirectional LSTM based potential function is detailed as follows.
First, each word $x_i$ ($i=1,\ldots,l$) in a sentence is mapped to an embedded vector $e_i \in R^d$.
Then the word embedding vectors are fed into a bidirectional LSTM to extract the long-range sequential features from the forward and backward contexts.
Denote by $h_{f,i}, h_{b,i} \in R^d$ the hidden vectors of the forward and backward LSTMs respectively at position $i$.
Finally, we calculate the inner product of the hidden vector of the forward LSTM at current position and the embedding vector at the next position,
and calculate the inner product of the hidden vector of the backward LSTM at current position and the embedding vector at the pervious position (dash line in \figref{fig:rnn-trf}).
The potential function $\phi(x^l;\theta)$ is computed by summating all the inner products:
    \begin{equation}\label{eq:phi}
      \phi(x^l;\theta) = \sum_{i=1}^{l-1} h_{f,i}^T e_{i+1} + \sum_{i=2}^{l} h_{b,i}^T e_{i-1}
    \end{equation}
where $\theta$ denotes all the parameters in the neural network.

\subsection{Noise-contrastive estimation (NCE)}

Noise-contrastive estimation (NCE) is proposed in \cite{nce} for learning unnormalized statistical models.
Its basic idea is ``learning by comparison'', i.e. to perform nonlinear logistic regression to discriminate between data samples drawn from the data distribution and noise samples drawn from a known noise distribution.
An advantage of NCE is that the normalization constants can be treated as the normal parameters and updated together with the model parameters.

To apply NCE to estimate neural TRF LMs defined in \eqref{eq:trf}, we treat the logarithmic normalization constants $\log Z_l$, $l=1, \ldots, m$ as parameters and rewrite \eqref{eq:trf} in the following form:
    \begin{equation}\label{eq:trf-nce}
      p(l,x^l;\hat \theta) = \pi_l e^{\hat \phi(l,x^l; \hat \theta)}.
    \end{equation}
Here $\hat \phi(l,x^l;\hat{\theta}) = \phi(x^l;\theta) - \log Z_l$,
and $\hat \theta = (\theta, \log Z_1,$ $\ldots, \log Z_m)$ consists of the parameters of the potential function and the normalization constants, which can be estimated together in NCE.
There are three distributions involved in NCE -- the true but unknown data distribution denoted by $p_d(l,x^l)$, the model distribution $p(l,x^l;\hat\theta)$ in \eqref{eq:trf-nce} and a fixed noise distribution denoted by $p_n(l,x^l)$, which is defined as a bigram LM in \cite{Bin2018}.

Consider the binary classification of a sentence $(l,x^l)$ coming from two classes - from the data distribution ($C=0$) and from the noise distribution ($C=1$), where $C$ is the class label.
Assume that the ratio between the prior probabilities is $1:\nu$, and the class-conditional probability for $C=0$ is modeled by $p(l,x^l;\hat\theta)$.
Then the posterior probabilities can be calculated respectively as follows:
    \begin{align}
    P(C=0|l, x^l; \hat \theta) &= \frac{p(l, x^l; \hat \theta)}{p(l, x^l; \hat \theta) + \nu p_n(l, x^l)}  \label{eq:pc0} \\
    P(C=1|l, x^l; \hat \theta) &= 1 - P(C=0|l, x^l; \hat \theta)  \label{eq:pc1}
    \end{align}
NCE estimates the model distribution by maximizing the following conditional log-likelihood:
    \begin{equation} \label{eq:j}
    \begin{split}
    J(\hat \theta) = \sum_{l=1}^{m} \sum_{x^l} p_d(l,x^l) \log P(C=0|l,x^l;\hat\theta) + \\
                  \nu \sum_{l=1}^{m} \sum_{x^l} p_n(l,x^l)\log P(C=1|l,x^l;\hat\theta)
    \end{split}
    \end{equation}
$J(\hat \theta)$ is the summation of two expectations.
The first is the expectation with respect to (w.r.t.) the data distribution $p_d(l,x^l)$, which can be approximated by randomly selecting sentences from the training set.
The second is the expectation w.r.t. the noise distribution $p_n(l,x^l)$, which can be computed by drawing sentences from the noise distribution itself.

Denote by $D$ and $B$ the data set and the sample set at current iteration,
and by $|D|$ and $|B|$ the number of sentences in $D$ and $B$, respectively, satisfying $\nu = |B|/|D|$.
The gradient with respect to $\hat\theta$ can be computed as follows:
    \begin{equation} \label{eq:grad}
    \begin{split}
    \frac{\partial J(\hat\theta)}{\partial \hat\theta} =
                    \frac{1}{|D|} \sum_{(l,x^l) \in D}P(C=1|l,x^l;\hat\theta) \frac{\partial \hat\phi(l, x^l; \hat\theta)}{\partial \hat\theta} \\
                    - \frac{\nu}{|B|} \sum_{(l,x^l) \in B} P(C=0|l,x^l;\hat\theta) \frac{\partial \hat\phi(l, x^l; \hat\theta)}{\partial \hat\theta}
    \end{split}
    \end{equation}
The gradient of the potential function $\hat\phi(l, x^l;\hat\theta)$ w.r.t. the parameters $\theta$ can be efficiently computed through the back-propagation algorithm.
Then any gradient method can be used to optimize the parameters and normalization constants, such as stochastic gradient descent (SGD) or Adam \cite{adam}.

\section{Dynamic noise-contrastive estimation}
\label{sec:dnce}

The application of NCE to train neural TRF LMs is encouraging as introduced in \cite{Bin2018}. However, there still exist two problems.
First, reliable NCE needs a large $\nu$, especially when the noise distribution is not close to the data distribution.
However, the time and memory cost for gradient calculation in \eqref{eq:grad} are almost linearly increased with $\nu$.
In \cite{Bin2018}, the noise distribution is defined by a bigram LM, which is far from the data distribution. For each real sentence, $\nu=20$ noise sentences are drawn from the bigram LM, which is highly undesirable.
Second, the expectation w.r.t. the data distribution in \eqref{eq:j} is approximated by the expectation w.r.t. the empirical distribution (namely the training set), which is rather sparse considering the high-dimensionality of sentences.
The model estimated by NCE is thus easily overfitted to the empirical distribution.
In this section, we propose dynamic noise-contrastive estimation (DNCE) to address the above two problems respectively.

In DNCE, we define a dynamic noise distribution $p_n(l,x^l;\mu)$ with parameter $\mu$, which is optimized simultaneously by minimizing the KL divergence between the noise distribution and the data distribution, i.e.
    \begin{equation}\label{eq:kl}
      \min_{\mu} KL(p_d||p_n) \Leftrightarrow \max_{\mu} \sum_{l=1}^{m} \sum_{x^l} p_d(l,x^l) \log p_n(l,x^l;\mu)
    \end{equation}
The purpose is to push the noise distribution to be close to the data distribution,
so that we can achieve reliable model estimation even using a small $\nu$.
Then we estimate the model distribution by maximizing the following conditional log-likelihood:
    \begin{equation}\label{eq:j2}
    \begin{split}
      \hat J(\hat \theta) = \sum_{l=1}^{m} \sum_{x^l} q(l,x^l;\mu) \log P(C=0|l,x^l;\hat\theta,\mu) + \\
                  \nu \sum_{l=1}^{m} \sum_{x^l} p_n(l,x^l;\mu)\log P(C=1|l,x^l;\hat\theta,\mu)
    \end{split}
    \end{equation}
where $P(C=0|l,x^l;\hat\theta,\mu)$ and $P(C=1|l,x^l;\hat\theta,\mu)$ are defined as \equref{eq:pc0} \equref{eq:pc1} by replacing the fixed noise distribution $p_n(l,x^l)$ with the dynamic noise distribution $p_n(l,x^l;\mu)$.
$q(l,x^l;\mu) = \alpha p_d(l,x^l) + ( 1 - \alpha) p_n(l,x^l;\mu)$ is the interpolation of the data distribution and the noise distribution, and $0 < \alpha < 1$ is the interpolating factor.
Compared with \eqref{eq:j}, \eqref{eq:j2} optimizes a discriminator between the noise distribution $p_n(l,x^l;\mu)$ and the interpolation of the data distribution $p_d(l,x^l)$ and the noise distribution $p_n(l,x^l;\mu)$.
The following proposition shows the theoretical consistency of DNCE learning in the nonparametric limit.

\begin{prop}
Suppose that an arbitrarily large number of real sentences can be drawn from $p_d(l,x^l)$, and the noise distribution $p_n(l,x^l;\mu)$ and the model distribution $p(l,x^l;\hat\theta)$ have infinite capacity.
Then we have (i) the KL divergence $KL(p_d||p_n)$ in \equref{eq:kl} can be minimized to attain zero. (ii) If $KL(p_d||p_n)$ attains zero at $\mu^*$, and the conditional log-likelihood \equref{eq:j2} attains a maximum at $\hat\theta^*$,
then we have $p(l,x^l;\hat\theta^*) = p_n(l,x^l;\mu^*) = p_d(l,x^l)$.
\end{prop}

\begin{proof}
  By the property of KL divergence, we have $p_n(l,x^l;\mu^*) = p_d(l,x^l)$.
  By the conclusion of NCE in \cite{nce}, with fixed $\mu^*$, \equref{eq:j2} has the only extremum at
  $p(l,x^l;\hat\theta^*) = q(l,x^l;\mu^*) = \alpha p_d(l,x^l) + ( 1 - \alpha) p_n(l,x^l;\mu^*) = p_d(l,x^l)$.
\end{proof}

About DNCE, we provide the following comments.
\begin{enumerate}
  \item Intuitively, as the noise distribution $p_n$ converges to the data distribution,
      using the interpolated distribution $q(l,x^l;\mu)$ will increase the number of data sentences by adding sampled sentences drawn from the noise distribution. This could avoids the neural TRF model to be overfitted to the sparse training set.
  \item Theoretically, it is feasible to optimize the noise distribution independently and then plug it into \eqref{eq:j2} to estimate the model distribution.
        However, it is found in our experiments that simultaneously optimizing the noise distribution and the model distribution is more stable, which is also theoretically-correct.
  \item In practice, the noise distribution is defined as $p_n(l,x^l;\mu)=\pi_l p_{n,l}(x^l;\mu)$.
        $\pi_l$ is the prior length probability which is the same as in \equref{eq:trf} and usually set to the empirical length distribution.
        This definition ensures that the length distributions of the noise sentences and of the data sentences in the training set are the same.
        $p_{n,l}(x^l;\mu)$ is defined by a simple LSTM LM. It is straightforward to calculate the probability of a whole sentence and draw sentences from this noise distribution.
\end{enumerate}

The DNCE training algorithm is summarized as follows.
At each iteration, $K_D$ data sentences are  drawn from the training set, denoted by $D^{(t)}$,
which is used to update the noise distribution by
    \begin{equation}\label{eq:updatemu}
      \mu^{(t)} = \mu^{(t-1)} + \gamma_{\mu} \left\{ \frac{1}{K_D} \sum_{(l,x^l) \in D^{(t)}} \frac{\partial}{\partial \mu} \log p_n(l,x^l;\mu) \right\},
    \end{equation}
where $\mu^{(t)}$ and $\mu^{(t-1)}$ denote the estimated parameter $\mu$ of the noise distribution at current iteration $t$ and previous iteration $t-1$ respectively, and $\gamma_{\mu}$ is the learning rate.
Then, two sets of noise sentences are generated from the noise distribution $p_n(l,x^l; \mu^{(t)})$, denoted by $B^{(t)}_1$ and $B^{(t)}_2$,
whose sizes satisfy $|B^{(t)}_1| = \frac{1-\alpha}{\alpha} K_D$ and $|B^{(t)}_2| = \frac{\nu}{\alpha} K_D$ respectively.
As a result, the union of $D^{(t)}$ and $B^{(t)}_1$ can be viewed as the samples drawn from the interpolated distribution $q(l,x^l;\mu) = \alpha p_d(l,x^l) + ( 1 - \alpha) p_n(l,x^l;\mu)$.
The update of model parameter $\hat\theta$ is as follows:
    \begin{equation}\label{eq:updatetheta}
    \small
      \hat\theta^{(t)} = \hat\theta^{(t-1)} + \gamma_{\hat\theta} Adam \left\{  g^{(t)}(\hat\theta) \right\},
    \end{equation}
where $g^{(t)}(\hat\theta)$ is the gradient of the conditional log-likelihood \eqref{eq:j2} w.r.t. $\hat\theta$ at current iteration $t$, i.e.
    \begin{equation}\label{eq:g_theta}
    \begin{split}
       g^{(t)}(\hat\theta) = \frac{\alpha}{K_D} \sum_{(l,x^l) \in D^{(t)} \bigcup B^{(t)}_1} P(C=1|l,x^l;\hat\theta,\mu) \frac{\partial \hat\phi(l,x^l; \hat\theta)}{\partial \hat\theta} \\
                    - \frac{\alpha}{K_D} \sum_{(l,x^l) \in B^{(t)}_2} P(C=0|l,x^l;\hat\theta,\mu) \frac{\partial \hat\phi(l,x^l; \hat\theta)}{\partial \hat\theta}.
    \end{split}
    \end{equation}
$\hat\theta^{(t)}$ and $\hat\theta^{(t-1)}$ denote the estimated $\hat\theta$ at current iteration $t$ and previous iteration $t-1$ respectively, $Adam$ denotes the Adam method \cite{adam} and $\gamma_{\hat\theta}$ is the learning rate.
As $\hat\theta = (\theta, \log Z_1, \ldots, \log Z_m)$, the parameters of the potential function and the normalization constants of neural TRF LMs can be jointly estimated in DNCE.

\section{Experiments}
\label{sec:exps}

Three speech recognition experiments are conducted to evaluate the neural TRF LMs with DNCE training, all in the form of rescoring using different LMs.
First, LMs are trained on Wall Street Journal (WSJ) portion of Penn Treebank (PTB), and evaluated by word error rate (WER) on WSJ'92 test data, with the same experimental setup as in \cite{Bin2017, Bin2017ASRU}.
We compare the performance of NCE and DNCE and find that DNCE is able to avoid the overfitting problem and achieves a lower WER with a small sample number $\nu = 1$.
Then, we conduct the experiment on HKUST Chinese dataset to examine the language independence in applying neural TRF LMs.
Finally, neural TRF LMs are trained on Google one-billion word corpus\footnote{
https://github.com/ciprian-chelba/1-billion-word-language-modeling-benchmark},
which contains about 0.8 billion words with a vocabulary of about 568 K words.
The CPUs used in the following experiments are Intel Xeon E5 (2.00 GHz) and the GPUs are NVIDIA GeForce GTX 1080Ti.

\subsection{Neural TRF LMs on PTB dataset}
\label{sec:exp-ptb}

\begin{table*}[t]
\centering
\begin{tabular}{l|c|c|c|l|l}
		\hline
		Model               &   PPL     &   WER (\%)     &   \#param (M)  & Training time  & Inference time \\
		\hline
		KN5                 &   141.2  &   8.78        &   2.3        & 22 seconds (1 CPU)  &  0.06 seconds (1 CPU) \\
		LSTM-2$\times$200   &   113.9   &   7.96    &   4.6          & 1.7 hours (1 GPU)   &  6.36 seconds (1 GPU) \\
		LSTM-2$\times$650   &   84.1    &   7.66    &   19.8        & 7.5 hours (1 GPU)   &  6.36 seconds (1 GPU) \\
		LSTM-2$\times$1500  &   78.7    &   7.36    &   66.0        & 1 day (1 GPU)       &  9.09 seconds (1 GPU) \\
        \hline
		discrete TRF in \cite{Bin2017}  & $\geq$130    &   7.90    &   6.4        & 1 day (8 CPUs)    &  0.16 seconds (1 CPU) \\
		neural TRF in \cite{Bin2017ASRU}             &   $\geq$37  &   7.60   &   4.0      & 3 days (1 GPU)      &  0.40 seconds (1 GPU) \\
        \hline
		this paper            &   $\sim$66  &   7.40   &   2.6      &  1 days (1 GPU)     &  0.08 seconds (1 GPU)\\
		\hline
\end{tabular}
\caption{Speech recognition results of various LMs, trained on WSJ portion of PTB dataset.
        ``PPL'' is the perplexity on PTB test set.
		``WER'' is the rescoring word error rate on WSJ'92 test data.
		``\#param'' is the number of parameter numbers (in millions).
        ``Training time'' denotes the total training time for a LM.
		``Inference time'' denotes the average time of rescoring the n-best list for each utterance.}
\label{tab:ptb}
\end{table*}

\begin{table*}
  \centering
  \begin{tabular}{l|c|c|c|l|l}
		\hline
		Model               & Valid (\%)  &   Test (\%)     &   \#param (M)  & Training time  & Inference time \\
		\hline
		KN5                 &  27.69   &  28.48  &  3.5  &  8 seconds (1 CPU) & 0.004 second (1 CPU) \\
        LSTM-2$\times$200          &  26.98   &  27.60  &  2.2  &  0.5 hour (1 GPU) & 0.048 second (1 GPU) \\
		\hline
        neural TRF       &   26.32  &   27.72   &   1.4      &  1 day (1 GPU) &  0.009 second (1 GPU)\\
        \hline
        KN5$+$LSTM     &  26.53  &   27.36  & & & \\
        KN5$+$nerual TRF  & 26.32  & 27.30 & & &  \\
        LSTM$+$neural TRF &  25.89  &   26.91   & & & \\
        KN5$+$LSTM$+$neural TRF & 25.96 & 26.87  & & & \\
        \hline
  \end{tabular}
  \caption{Speech recognition results on HKUST Chinese dataset.
           ``Valid'' is the character error rate (CER) on the valid set and ``Test'' is the CER on the test set.
           ``$+$'' denotes the log-linear interpolation with equal weights.
       Other columns has the same meanings as in \tabref{tab:ptb}.}
  \label{tab:hkust}
  \vspace{-10pt}
\end{table*}

In this section, the LM training corpus is the Wall Street Journal (WSJ) portion of Penn Treebank (PTB) dataset.
Section 0-20 are used as the training set (about 930 K words), session 21-22 as the development set (about 74 K words) and section 23-24 as the test set (about 82 K words).
The vocabulary is limit to 10 K words, including a special token ``$\langle$unk$\rangle$'' denoting the word not in the vocabulary.
For evaluation in terms of speech recognition WER,
various LMs trained on PTB training and development sets are applied to rescore the 1000-best lists from recognizing WSJ'92 test data (330 utterances).
For each utterance, the 1000-best list of candidate sentences are generated by the first-pass recognition using the Kaldi toolkit\footnote{http://kaldi.sourceforge.net/} with a DNN-based acoustic models.
The oracle WER of the 1000-best list is 0.93\%.
This setting is the same as that used in \cite{Bin2017, Bin2017ASRU}.

\begin{figure}[t]
\begin{minipage}[b]{0.48\linewidth}
  \centering
  \centerline{\includegraphics[width=\linewidth]{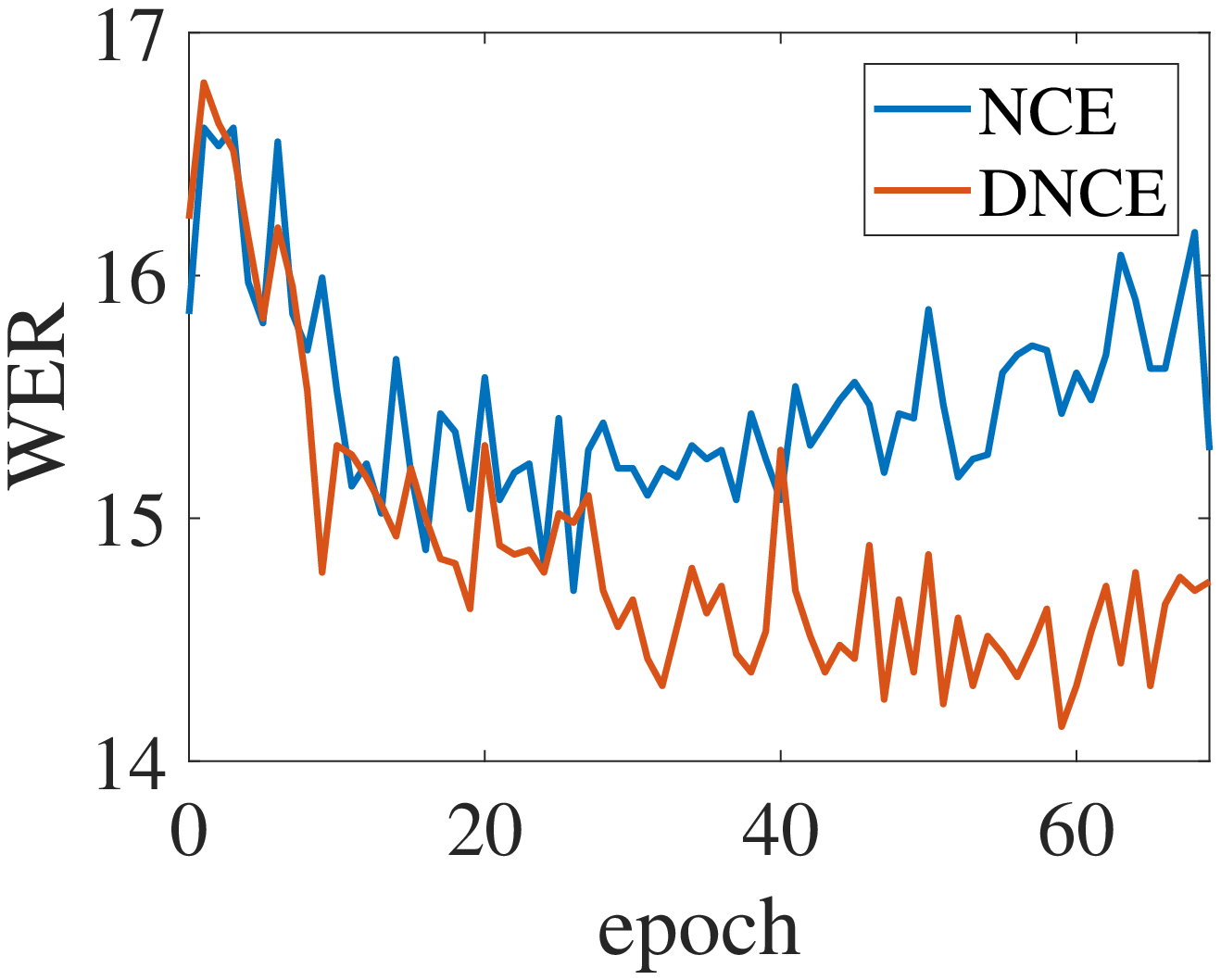}}
  \centerline{(a)}
\end{minipage}
\hfill
\begin{minipage}[b]{0.48\linewidth}
  \centering
  \centerline{\includegraphics[width=\linewidth]{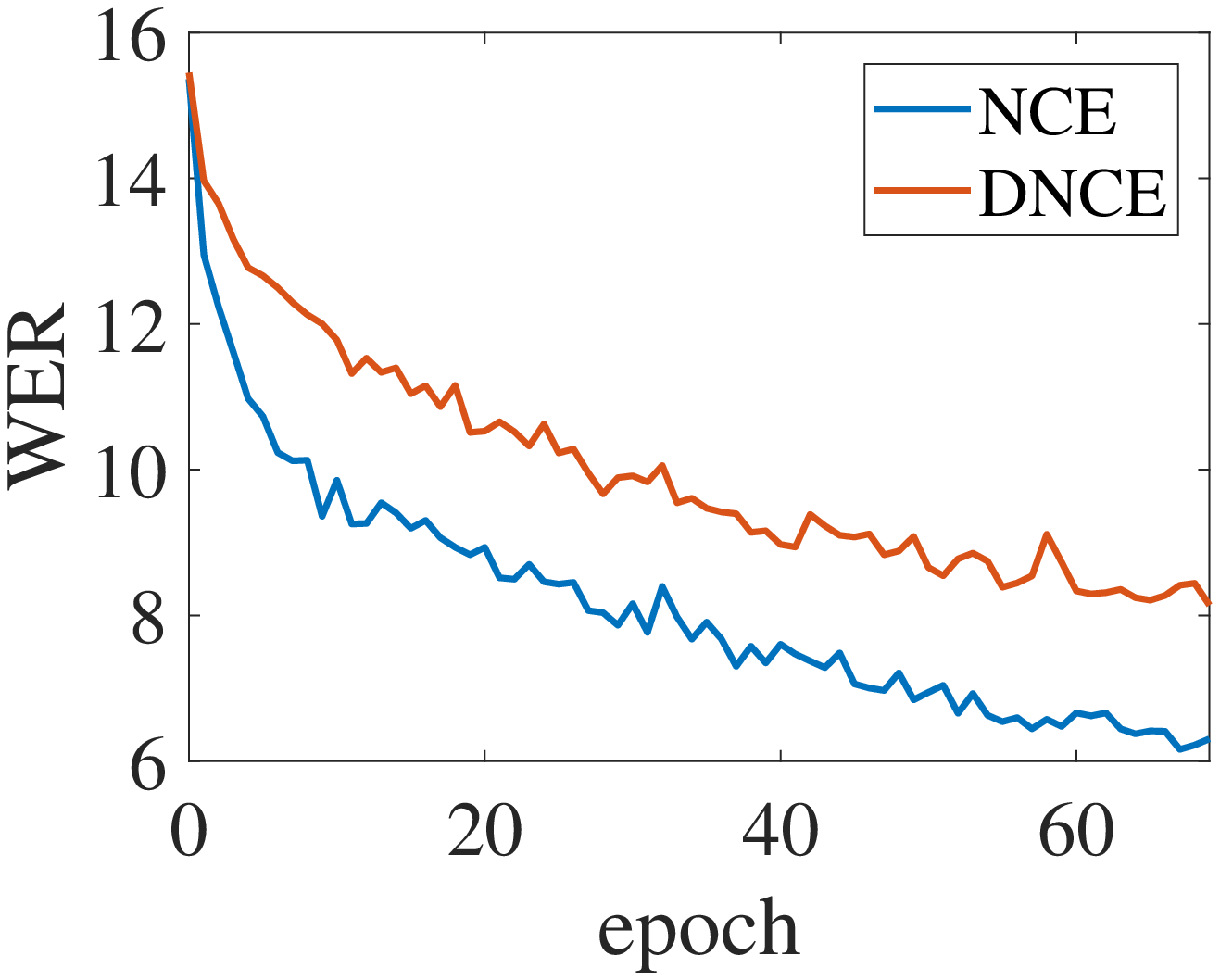}}
  \centerline{(b)}
\end{minipage}
\caption{(a) Rescoring WERs on the WSJ'92 1000-best list calculated with only LM scores.
(b) Rescoring WERs on the fake n-best list calculated with only LM scores.}
\label{fig:ptb}
\vspace{-10pt}
\end{figure}

A neural TRF LM defined in \secref{sec:trf} is trained by DNCE method, with the embedding size $d=200$ and the bidirection LSTM containing 1 hidden layer and $d=200$ hidden units.
The maximum length in the neural TRF LM is $m=82$, which is equal to the maximum length of sentences in the training and development set.
The noise distribution $p_n(l,x^l;\mu)$ is defined by a simple LSTM LM with 1 hidden layer and 200 hidden units.
At each iteration, DNCE draws $K_D=100$ data sentence from the training set as $D^{(t)}$ to update the noise distribution $p_n(l,x^l;\mu)$ based on \eqref{eq:updatemu},
and then generates two sample set of size $\frac{1-\alpha}{\alpha} K_D=100$ and $\frac{\nu}{\alpha} K_D=200$ (i.e. $\alpha=0.5$, $\nu=1$) from the noise distribution as $B^{(t)}_1$ and $B^{(t)}_2$ to update the neural TRF LM based \eqref{eq:updatetheta}.
The learning rate in \eqref{eq:updatemu} and \eqref{eq:updatetheta} are set to $\gamma_{\mu} = 0.01$ and $\gamma_{\hat\theta}=0.1$ and fixed during the training process.
All the NN parameters in the neural TRF LM and the noise LSTM LM are initialized randomly within an interval from -0.1 to 0.1 and the normalization constants of neural TRF LMs are initialized to $\log Z_l = l \times |V|$ ($l=1,\dots,m$), where $|V|$ is the vocabulary size.
We stop the training once the log-likelihood on the PTB development set does not increase significantly, resulting in 70 training epoches.

To compare the NCE and DNCE training method,
we create a fake n-best list, by randomly selecting 3000 data sentences from the training set and randomly introducing several substitution, insertion or deletion errors in each sentence.
LMs are used to rescore the fake n-best list and the WSJ'92 1000-best list, and sentences with the largest LM probabilities are treated as the recognition results and used to calculate the word error rates (WERs).
This is to exclude the influence of the acoustic models and mainly to evaluate the performance of LMs.
The convergence of WERs on the fake n-best list and the WSJ'92 1000-best list are shown in \figref{fig:ptb}.
On the fake nbest-list (\figref{fig:ptb}(b)), the WER of NCE reduces faster that DCNE.
However NCE performs worse than DNCE in the real test set - WSJ'92 1000-best list (\figref{fig:ptb}(a)),
and the WER of NCE begins to increase after about 20 training epoches.
In contrast, DNCE effectively avoid the overfitting problem by adding the noise sentences into the data sentences in \eqref{eq:j2}, and as a result DNCE achieves a lower WER on the real test set (\figref{fig:ptb}(a)).

The final results of various LMs are shown in \tabref{tab:ptb}, including a 5-gram LM with modified Kneser-Ney smoothing \cite{chen1999empirical} (denoted by ``KN5'') and three LSTM LMs with 2 hidden layers and 200, 600, 1500 hidden units per layer respectively (denoted by ``LSTM-2$\times$200'', ``LSTM-2$\times$600'' and ``LSTM-2$\times$1500''), which correspond to the small, medium and large LMs in \cite{lstmdropout}.
As the vocabulary size is not large, all LSTM LMs use the standard softmax output layers to ensure the best model performance.
We also show the results of TRF LMs in the previous studies in \cite{Bin2017, Bin2017ASRU}.
The conclusions of this experiments are summarized as follows.

First, the proposed DNCE method successfully train the neural TRF LM with a small noise sample number $\nu=1$.
Compared to the neural TRF LM in \cite{Bin2017ASRU}, which is trained by the AugSA method, DNCE reduces the training time from 3 days to 1 days and reduces the WER from 7.60\% to 7.40\%.
Compared to the NCE used in \cite{Bin2018}, DNCE successfully reduces the noise sample number from $\nu=20$ to $\nu=1$.
Second, the neural TRF LM outperforms the 5gram LMs significantly with about 15.7\% relative WER reduction.
Finally, compared with the large LSTM LM with 2 hidden layers and 1500 hidden units per layer (denoted by ``LSTM-2$\times$1500''),
the neural TRF LM achieves a close WER with only about 4.0\% parameters and is 114x faster in rescoring sentences in n-best lists.

\begin{table*}
  \centering
  \begin{tabular}{l|c|c|l|l}
		\hline
		Model               &   WER (\%)     &   \#param (M)  & Training time  & Inference time \\
		\hline
		KN5                 &  6.13  &  133  &  2.5 hours (1 CPU) & 0.491 second (1 CPU) \\
        LSTM-2$\times$1024         &  5.55  &  191  &  6 days (2 GPUs) & 0.909 second (2 GPUs) \\
		\hline
        neural TRF          &   5.47   &   114      &  14 days (2 GPUs) &  0.017 second (1 GPU)\\
        \hline
        KN5$+$LSTM     &   5.38  & & & \\
        KN5$+$nerual TRF  & 5.51 & & &  \\
        LSTM$+$neural TRF &   5.25   & & & \\
        KN5$+$LSTM$+$neural TRF & 5.06  & & & \\
        \hline
  \end{tabular}
  \caption{Speech recognition results of various LMs trained on Google one-billion benchmark.
 ``$+$'' denotes the log-linear interpolation with equal weights.
  Other columns has the same meanings as in \tabref{tab:ptb}.}
  \label{tab:google1b}
  \vspace{-10pt}
\end{table*}

\subsection{Nerual TRF LMs on HKUST dataset}
\label{sec:exp-hkust}

In this section, we perform the speech recognition experiments on HKUST Chinese dataset \cite{hkust}.
Various character-based LMs are trained on the training corpus consisting of about 2.4M characters, with a vocabulary of about 4000 Chinese characters,
and are then used to rescore the 100-best list generated by the Kaldi scripts based on a LF-MMI \cite{lfmmi} acoustic model.
From all recognized utterances, we randomly select 1082 (about 20\%) utterances as the valid set and treat the rest 4331 utterances as the test set.
All the hyper-parameters including the learning rate and the training epoch number are tuned on the valid set.

The configurations of neural TRF LMs and the DNCE training method are the same to the experiment in \secref{sec:exp-ptb},
except that the maximum length is set to $m=43$ and the noise sample number is increased from $\nu=1$ to $\nu=4$.
The results of the our neural TRF LM and the baseline LMs are shown in \tabref{tab:hkust}, from which there are several comments.
First, our neural TRF LM outperforms the classical 5-gram LM (denoted by ``KN5'') with relative CER reduction 2.7\%, and performs close to ``LSTM-2$\times$200'' (a LSTM LM with 2 hidden layers, 200 hidden units per layer and a standard softmax output layer) with 63\% parameters.
Second, even though the vocabulary in this experiment is small, containing only about 4000 characters, our neural TRF LM is still 5 times faster than LSMT LM in rescoring sentences.
Third, the lowest WER 26.87\% is achieved by combining the neural TRF LM with the 5-gram LM and the LSTM LM (denoted by ``KN5$+$LSTM$+$neural TRF'').
Finally, together with the previous experiment on PTB dataset, these results demonstrate the language independence in applying neural TRF LMs.

\subsection{Neural TRF LMs on Google one-billion word benchmark}

In this section, we examine the scalability of neural TRF LMs on Google one-billion word benchmark.
The training set contains about 0.8 billion words.
We map the words whose counts less than 4 to the token ``$\langle$unk$\rangle$'', and obtain a vocabulary of about 568 K words.
Various LMs trained on the training set are used to rescore the 1000-best list of the WSJ'92 test set, which is same as in \secref{sec:exp-ptb}.

Training LMs on corpus with large vocabulary is challenging.
For n-gram LMs, a 5-gram LM with cutoff setting of ``00225'' (denoted by ``KN5'') is trained using the SRILM toolkit\footnote{http://www.speech.sri.com/projects/srilm/} by separately counting over the split training files and then merging the counts together.
A LSTM LM which uses the embedding size of 256, 2 hidden layers and 1024 hidden units per layer (denoted by ``LSTM-2$\times$1024'') is trained using the adaptive softmax strategy proposed in \cite{grave2016efficient}.
Adam \cite{adam} is used to train the LSTM LM with a learning rate initialized to $10^{-3}$ and halved per epoch.
The final training epoch number is 4.

A neural TRF LM defined in \secref{sec:trf}, which uses the embedding size $d=200$ and the bidirectional LSTM containing 1 hidden layer and $d=200$ hidden units, is trained by DNCE.
The maximum length of the neural TRF LM is set to $m=60$ and sentences in the training set longer than 60 are omitted.
The noise distribution $p_n(l,x^l;\mu)$ is defined by a simple LSTM LM with 1 hidden layers and 200 hidden units.
As the direct softmax calculation on the whole vocabulary is infeasible in this experiment, we introduce the shortlist strategy, which uses the LSTM to predict the first 10 K frequent words and uses a unigram to predict the rest of words.
For DNCE, the interpolation factor is set to $\alpha=2/3$ and the noise sample number is set to $\nu=4$.
At each iteration, $K_D=100$ data sentences are drawn from the training set as $D^{(t)}$ to update the noise distribution $p(l,x^l;\mu)$ based on \eqref{eq:updatemu}.
Two sample sets of size $\frac{1-\alpha}{\alpha} K_D=50$ and $\frac{\nu}{\alpha} K_D=600$ are drawn respectively from the noise distribution as $B^{(t)}_1$ and $B^{(t)}_2$,  and used to update the neural TRF LM based on \eqref{eq:updatetheta}.
The settings of learning rates and the initial parameters are the same as that in \secref{sec:exp-ptb}.
We stop the training once the log-likelihood on the development set does not increase significantly, resulting in 5 training epoches.

The speech recognition WERs are shown in \tabref{tab:google1b}.
We have the following conclusions.
\begin{itemize}
  \item DNCE exhibits the capability to handle large corpus with large vocabulary. With a simple noise LM and a small noise sample number ($\nu=4$), DNCE can be used to train neural TRF LMs effectively on Google one-billion word benchmark corpus with 0.8 billion words and a vocabulary of about 568 K words.
  \item The neural TRF LM trained by DNCE outperforms the classical 5-gram LM with modified Kneser-Ney smoothing \cite{chen1999empirical} (denoted by ``KN5'') with relative WER reduction 10.8\%, and performs slightly better than ``LSTM-2$\times$1024'' with relative WER reduction 1.4\%.
  \item The neural TRF LM shows its distinctive advantage in rescoring sentences, which is about 54 times faster than the LSTM LM with adaptive softmax.
  \item The lowest WER 5.06\% is achieved by combining the neural TRF LM with the 5-gram LM and the LSTM LM, with 17.5\% relative WER reduction over ``KN5'' and 5.9\% over ``KN5+LSTM''.
\end{itemize}

\section{Conclusions}
\label{sec:conclusion}

In this paper, we further investigate the training methods of neural TRF LMs, and propose the dynamic noise-contrastive estimation (DNCE).
The following improvements enable the successful and efficient training of neural TRF LMs on Google one-billion word benchmark, which contains about 0.8 billion English words with a vocabulary of about 568 K.
\begin{itemize}
  \item Instead of using a fixed noise distribution in NCE, a dynamic noise distribution is introduced in DNCE and  trained simultaneously by minimizing the KL divergence between the noise distribution and the data distribution, in addition to training the model.
  This helps to significantly cut down the noise sample number and reduce the training cost.
  \item DNCE discriminates between sentences generated from the noise distribution and sentences generated from the interpolation of the data distribution and the noise distribution.
  Using the interpolated distribution alleviates the overfitting problem caused by the sparseness of the training set.
\end{itemize}

There are some interesting directions for future research.
First, based on the success of neural TRFs in language modeling and speech recognition,
it is worthwhile to investigate the application of neural TRFs in other sequential and trans-dimensional data modeling tasks.
Second, although we mainly apply DNCE to neural TRF model training in this paper, it can be seen that DNCE is a general improvement over NCE by introducing a dynamic noise distribution and using the interpolation of the data distribution and the dynamic noise distribution in the discriminator.
Therefore, DNCE can be applied in other tasks wherever NCE is used.
Using DNCE could significantly reduce the noise sample number and alleviate the overfitting problem.

\bibliographystyle{IEEEbib}
\bibliography{RF}

\end{document}